\documentclass{article}
\usepackage{amsmath,graphicx,mlspconf}

\usepackage[caption=false,font=normalsize,labelfont=sf,textfont=sf]{subfig}
\usepackage{algorithmicx}
\usepackage[linesnumberedhidden,ruled,vlined,algo2e]{algorithm2e}
\usepackage{hyperref}

\usepackage{amsmath}
\usepackage{amssymb}
\usepackage{mathtools}
\usepackage{amsthm}
\usepackage{nicefrac}
\usepackage{bm}

\usepackage{cite}

\usepackage{dsfont}\usepackage{array}\usepackage{mathrsfs}
\usepackage{comment}

\theoremstyle{plain}
\newtheorem{theorem}{Theorem}[section]
\newtheorem{proposition}[theorem]{Proposition}

\theoremstyle{definition}

\newtheorem{assumption}[theorem]{Assumption}
\theoremstyle{remark}

\newcommand{\balpha}{{\boldsymbol \alpha}}

\newcommand{\bA}{{\boldsymbol A}}

\newcommand{\bw}{{\boldsymbol w}}
\newcommand{\bx}{{\boldsymbol x}}
\newcommand{\by}{{\boldsymbol y}}

\newcommand{\E}{{\mathbb E}}
\newcommand{\R}{{\mathbb R}}

\copyrightnotice{979-8-3503-2411-2/23/\$31.00 {\copyright}2023 IEEE}

\toappear{2023 IEEE International Workshop on Machine Learning for Signal Processing, Sept.\ 17--20, 2023, Rome, Italy}

\title{Distributed Dual Coordinate Ascent with Imbalanced Data on a General Tree Network}
\name{
    Myung Cho${^{\star}}^{\mathsection}$%
    \qquad Lifeng Lai$^{\dagger}$\thanks{Lifeng Lai's research is supported by National Science Foundation (NSF) under grant ECCS-2000415.}%
    \qquad Weiyu Xu$^{\ddagger}$\thanks{Weiyu Xu's research is supported by NSF under grant ECCS-2000425 and ECCS-2133205.}%
}
\address{%
    $^{\star}$ Department of Electrical and Computer Engineering, California State University, Northridge, CA, USA \\%
    $^{\mathsection}$ Department of Electrical and Computer Engineering, Penn State Behrend, Erie, PA, USA \\%
    $^{\dagger}$  Department of Electrical and Computer Engineering, University of California, Davis, CA, USA\\%
    $^{\ddagger}$  Department of Electrical and Computer Engineering, University of Iowa, Iowa City, IA, USA%
}

\begin{document}

\maketitle

\begin{abstract}
In this paper, we investigate the impact of imbalanced data on the convergence of distributed dual coordinate ascent in a tree network for solving an empirical loss minimization problem in distributed machine learning. To address this issue, we propose a method called delayed generalized distributed dual coordinate ascent that takes into account the information of the imbalanced data, and provide the analysis of the proposed algorithm. Numerical experiments confirm the effectiveness of our proposed method in improving the convergence speed of distributed dual coordinate ascent in a tree network.
\end{abstract}

\vspace{-0.5em}
\section{Introduction}
\label{sec:intro}
\vspace{-0.5em}
In the field of Machine Learning (ML) and Artificial Intelligence (AI), the use of large amounts of data, known as \textit{big data}, plays a crucial role in the performance of ML and AI techniques. Due to the exceptional performance of ML/AI with big data, these techniques are becoming increasingly popular and are being applied to many different applications.

However, in practice, processing ML/AI operations with big data poses many challenges such as limited hardware resources including storage space and processing power, and privacy and security issues. Specifically, due to limited storage space, big data is often stored in a distributed manner over a network, raising questions about how to deal with these distributed data when processing ML/AI operations. Furthermore, even though there are distributed algorithms that can process distributed data for ML/AI operations, communication to share intermediate results such as learning parameters can be a significant challenge due to constraints such as limited communication bandwidth, delay, and power. Thus, it is important to design efficient algorithms for distributed ML/AI processes that take into account these communication network constraints when dealing with distributed data.

To address the challenge of handling distributed data on a network, many research were conducted to develop efficient distributed ML/AI algorithms. Researchers in \cite{teng2018bayesian,ferdinand2018anytime,cong2017efficient,shi2018dag,ferdinand2017anytime}  studied synchronous Stochastic Gradient Decent (SGD). Synchronous Stochastic Dual Coordinate Ascent (SDCA) and its variations were investigated in \cite{hsieh2008dual,yang2013trading,jaggi2014communication,shalev2013stochastic,deng2021local,devarakonda2019avoiding,cho2019generalized,cho2021distributed}. Asynchronous SGD was studied in \cite{zhao2016fast,zhang2015fast,lian2018asynchronous}. Asynchronous SDCA was investigated for handling distributed data on a network in \cite{huo2016distributed,hsieh2015PASSCoDe,liu2014asynchronous,sun2017asynchronous}.

However, most of the research in this area has focused on designing distributed algorithms on a star network, which is a very simple network topology. It's worth noting that data are not always stored on a star network. In reality, networks can have various topologies such as star, tree, ring, bus, or mesh, etc. Additionally, in a network, some nodes may not directly communicate with a central node. In this case, a line of nodes may be considered as a virtual node directly connected to a central node to apply distributed algorithms for a star network. However, the system can easily suffer from significant communication delays caused by passing intermediate results through the line of nodes to a central node. Thus, it is important to design efficient distributed ML/AI algorithms by considering different network topologies. To address this problem, the researchers in \cite{cho2019generalized,cho2021distributed} have designed a distributed dual coordinate ascent for distributed ML process on a general tree network, and provided the convergence analysis of the algorithm on a general tree network under the assumption that the dataset is evenly distributed on a tree network. Since every connected network has its spanning tree, the distributed algorithm can be applied to various network topologies as long as there are no isolated nodes in the network.

In practice, however, due to different circumstances and situations in data acquisition such as different frequency of data acquisition in nodes, different sensitivity of sensors, and noise or bias during data acquisition, we can have imbalanced data on a network in distributed ML process. In this paper, we investigate the effect of imbalanced data on distributed dual coordinate ascent in a general tree network, and propose our method called a generalized distributed dual coordinate ascent on a general tree network to mitigate the effect of imbalanced data on processing distributed data on a tree network. 


\textbf{Notations:} We denoted the set of real numbers as $\R$. We use $[k]$ to denote the index set of the coordinates in the $k$-th coordinate block. For an index set $Q$, $|Q|$ and $\overline{Q}$ represent the cardinality of $Q$ and the complement of $Q$ respectively. We reserve bold letters to represent vectors and matrices. If an index set is used as a subscript of a vector (resp. matrix), it indicates the partial vector (resp. partial matrix) over the index set (resp. with columns over the index set). We use the superscript $(t)$ to denote the $t$-th iteration. For instance, $\balpha^{(t)}_{[k]}$ represents a partial vector $\balpha$ over the $k$-th block coordinate set at the $t$-th iteration. The superscript $\star$ is reserved to denote the optimal solution to an optimization problem.

\vspace{-0.5em}
\section{Problem Formulation}
\label{sec:prob}
\vspace{-0.5em} 
We perform an ML operation on a distributed dataset, denoted as $\{(\bx_i,y_i)\}_{i=1}^m$, where $\bx_i \in \R^d$ represents the $i$-th data point and $y_i$ is the measurement or label information associated with it. Our dataset is distributed in a tree-shaped network consisting of $K$ local workers, as illustrated in Fig. \ref{fig:generalDistSystem}. The $k$-th local worker holds a subset of the dataset, specifically $\{(\bx_i,y_i)\}_{i \in [k]}$ where $|[k]| < m$ and $k=1,2,...,K$. Our objective is to find the global optimal solution, $\bw^{\star}$, by solving the following regularized loss minimization problem with the distributed dataset: \\[-15pt]
\par\noindent\small
\begin{align}\label{prob:primal}
    \underset{\bw \in \R^d}{\text{minimize}}\; P(\bw) \triangleq \frac{\lambda}{2} \|\bw \|^2_2 + \frac{1}{m} \sum_{i=1}^{m} \ell_i(\bw^T \bx_i),\\[-20pt] \nonumber
\end{align}
\normalsize
where $\ell_i(\cdot)$, $i=1,2,...,m$, are loss functions, and $\lambda \geq 0$ is a tuning parameter. Depending on the loss functions, the optimization problem \eqref{prob:primal} can be a regression problem or a classification problem. For instance, when the loss function is $\ell_i(\bw^T \bx_i) = (\bw^T\bx_i - y_i)^2$ with measurement data $y_i \in \R$, the problem can be interpreted as a regression problem. When the loss function is $\ell_i(\bw^T\bx_i) = \max( 0, 1 - y_i (\bw^T \bx_i) )$ with label information $y_i \in \R$, \eqref{prob:primal} becomes a Support Vector Machine (SVM) classification problem. Furthermore, we assume that the data points are normalized to ensure that the $\ell_2$ norm of each $\bx_i$ is bounded, i.e., $\|\bx_i \|_2 \leq 1$, for $i=1,2,...,m$.

By considering the conjugate function of the loss function $\ell_i(a)$, defined as $\ell_i(a)=\sup_{b} ab - \ell_i^{*}(b)$, we can derive the following dual problem from the primal problem \eqref{prob:primal}: \\[-15pt]
\par\noindent\small
\begin{align}\label{prob:dual}
    \underset{\balpha \in \R^m}{\text{maximize}}\; D(\balpha) \triangleq -\frac{\lambda}{2} \|\bA\balpha\|^2_2 - \frac{1}{m} \sum_{i=1}^{m} \ell^{*}_i(-\alpha_i),\\[-20pt] \nonumber
\end{align}
\normalsize
where $\bA \in \R^{d \times m}$ is a data matrix whose $i$-th column is $\frac{1}{\lambda m} \bx_i$, and $\alpha_i$ is the $i$-th dual variable corresponding to the $i$-th datum $\bx_i$, $i=1, ..., m$. By defining $\bw(\balpha) \triangleq \bA \balpha$, we can have a duality gap as $P(\bw(\balpha)) - D(\balpha)$ which can be used as a measurable quantity for how close an estimated solution is to an optimal solution. Remark that the weak duality theorem \cite{boyd2004convex} ensures that for all $\bw$, $P(\bw) \geq D(\balpha)$. When we substitute $\bw$ with $\bw(\balpha)$, this condition still holds with $P(\bw(\balpha)) \geq D(\balpha)$. If we can find a variable $\balpha^{\star}$ such that $P(\bw(\balpha^{\star})) = D(\balpha^{\star})$, then, from the strong duality, $\balpha^{\star}$ and $\bw(\balpha^{\star})$ can be recognized as optimal solutions to the dual problem \eqref{prob:dual} and the primal problem \eqref{prob:primal} respectively. This paper aims to design a distributed algorithm that efficiently solves \eqref{prob:dual} while considering the information of imbalanced data, with the goal of mitigating the effect caused by data imbalance.


\begin{figure}[t]
    \centering
    \includegraphics[scale=0.55]{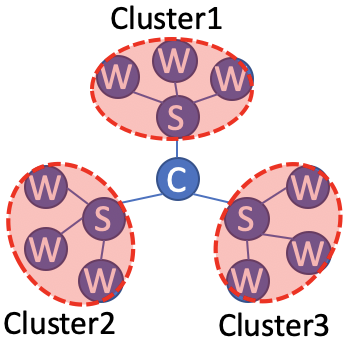}
    \vspace{-1em}
    \caption{\small{Illustration of a tree-structured network, which has two layers. In the network, a central station (root node) has three direct child nodes, i.e., sub-central node, denoted by $S$. Each sub-central node $S$ has three direct child nodes, i.e., local workers, denoted by $W$.}}
    \label{fig:generalDistSystem}
\end{figure}

\vspace{-1em}
\section{Review of Distributed Algorithms with Balanced or Imbalanced Data} 
\label{sec:review}
\vspace{-0.5em}
In previous studies, many researchers have addressed the issue of handling balanced or imbalanced data in the design of distributed algorithms for ML operations. For example, the authors  in \cite{wang2016fast,wang2021distributed} used the Alternating Direction Method of Multipliers (ADMM) technique to deal with various types of imbalanced data on a star network for classification problems. The authors in \cite{konevcny2015federated} considered a separable optimization problem on a star network with imbalanced data in the case when the number of local workers is larger than the number of data points. Unfortunately, most of the research focused on designing distributed algorithms on a star network, which is a very simple and special network topology. 


To deal with different types of network topologies, the researchers in \cite{cho2019generalized,cho2021distributed} extended the Distributed Dual Coordinate Ascent method  to a general tree network with $p$ layers (DDCA-Tree), where the root node is located in the 0-th layer and local workers are located in the $p$-th layer. Under the assumption that every node on the network has balanced data, the design of the distributed algorithm was studied. However, the design of an efficient distributed algorithm that considers imbalanced data in a general tree network has not been fully studied. This paper aims to address this research gap and investigate this topic.


\vspace{-1em}
\section{Generalized distributed dual coordinate ascent in a tree network}
\label{subsec:GDCA}
\vspace{-0.5em}
We aim to investigate the effect of imbalanced data on the convergence rate of DDCA-Tree when applied to a general tree network. We will also propose an enhanced version of DDCA-Tree that takes into account the imbalanced data to improve its performance. Specifically, we will address the following questions: (Q1) How does the convergence rate of DDCA-Tree differ when applied to imbalanced data? (Q2) How can we design DDCA-Tree to better handle imbalanced data by incorporating information about the imbalance?

To mitigate the effect of imbalanced data distribution, we can consider the information on the imbalanced data distribution in the accumulation of global parameters, which is the concept of Generalized distributed Dual Coordinate Ascent on a general tree network (GDCA-Tree). In GDCA-Tree, described in Algorithm \ref{alg:general_weighted}, we propose to use different weights for global parameters that local workers have in updating a global parameter at the $t$-th iteration, $\bw^{(t)}$. More specifically, in the accumulation of the global parameters, by considering the information of imbalanced data, we use a weighted sum updating scheme for global parameter $\bw^{(t)}$ in a general tree node $Q$ having $K$ child nodes as  \\[-15pt]
\par\noindent\small
\begin{align}
&\bw^{(t)} = \bw^{(t-1)} + \sum_{k=1}^{K} \beta_k  \bigtriangleup \bw_k, \label{eq:updating} \\[-20pt] \nonumber
\end{align}
\normalsize
where $\beta_k$ is a weight considering imbalanced data factor for the $k$-th updating parameter $\bigtriangleup \bw_k$, $\sum_{k=1}^K \beta_k \hspace{-0.2em}=\hspace{-0.2em}1$, $0 \leq \beta_k \leq 1$.

For the convergence analysis of GDCA-Tree, let us define the local sub-optimality gap, $\epsilon_{Q,k}$, at the $k$-th direct child node of a general tree node $Q$ as follows: \\[-15pt]
\par\noindent\small
\begin{align}\label{def:epsilon_Qk}
     & \epsilon_{Q,k}(\balpha) \triangleq \underset{\hat{\balpha}_{[Q;k]}}{\text{maximize}} \;D(\balpha_{[Q;1]},...,\hat{\balpha}_{[Q;k]},...,\balpha_{[Q;K]},\balpha_{\overline{Q}} ) \nonumber \\[-5pt]
     &\quad\quad\quad\quad\quad\quad - D(\balpha_{[Q;1]},...,\balpha_{[Q;k]},...,\balpha_{[Q;K]}, \balpha_{\overline{Q}}),\\[-20pt] \nonumber
\end{align}
\normalsize
where $[Q;k]$ is the set of indices of data points in the $k$-th direct child node of $Q$, $\balpha_{[Q;k]}$ is the partial vector of dual vector $\balpha \in \R^m$ that corresponds to the data points that the $k$-th direct child node of $Q$ has, and $\overline{Q}$ is complement of an index set $Q$. This local sub-optimality gap represents the maximum objective value gap that the $k$-th direct child node of $Q$ can achieve from the current $\balpha$ value. With this definition, we introduce the following assumption.
\vspace{-0.5em}
\begin{assumption} [Geometric improvement of GDCA-Tree at a direct child node] \label{asp:Local_Improve}
For a tree node $Q$ on the $i$-th layer, we assume that for any given $\balpha$ at the $k$-th direct child node of $Q$, the GDCA-Tree provides an update $\bigtriangleup \balpha_{[Q;k]}$ such that  \\[-15pt]
\par\noindent\small
\begin{align} \label{eq:Theta_i1}
    & \E[ \epsilon_{Q,k} (\balpha_{[Q;1]},...,\balpha_{[Q;k-1]}, \balpha_{[Q;k]}+\bigtriangleup \balpha_{[Q;k]},...,\balpha_{[Q;K]}, \balpha_{\overline{Q}})] \nonumber \\ 
   & \leq \Theta_{i+1} \cdot \epsilon_{Q,k}(\balpha),\\[-20pt] \nonumber
\end{align}
\normalsize
where $\Theta_{i+1} \in [0,1)$ is local improvement.
\end{assumption}
\vspace{-0.5em}
It represents that if this assumption holds, then, the expectation of the local sub-optimality gap at the $k$-th direct child node of a tree node $Q$ is decreased by the factor of local improvement $\Theta_{i+1}$. When using Local Stochastic Dual Coordinate Ascent (LocalSDCA), with a local dataset in a leaf node in the $p$-th layer of a tree network, it is possible to achieve the following proposition concerning the parameter $\Theta_{i+1}$:
\begin{proposition} (\cite{jaggi2014communication} ) \label{prop:LocalSDCA}
Assume that loss functions $\ell_i(\cdot)$ are $1/\gamma$-smooth. For a tree node $Q$ on the $(p-1)$-th layer, its direct child node is a leaf node. Then, for the leaf node $B$ in the $p$-th layer using LocalSDCA, Assumption \ref{asp:Local_Improve} holds with \\[-15pt]
\par\noindent\small
\begin{align}\label{eq:Theta}
\Theta_p = (1 - \frac{\lambda m \gamma}{ 1+ \lambda m \gamma} \cdot\frac{1}{m_B} )^{T_p},\\[-20pt] \nonumber
\end{align}
\normalsize
where $m_B$ is the size of data points that the leaf node $B$ has, and $T_p$ is the number of iterations in LocalSDCA. 
\end{proposition}

\begin{algorithm2e}[t]
  \caption{Generalized distributed Dual Coordinate Ascent in a general tree node $Q$ (GDCA-Tree) on the layer-$i$, $i=1,2,...,p-1$}
  \label{alg:general_weighted}
  \SetAlgoLined
{\small
   \KwIn{ $T_i \geq 1$, $\balpha_{Q}$,  $\bw$}
   \textbf{Initialization}: $\balpha_{[Q;k]}^{(0)} \leftarrow \balpha_{[Q;k]}$ for all direct child nodes $k$ of node $Q$ ,  $\bw^{(0)} \leftarrow \bw$   \par
   \For { $t=1$  \KwTo $T_i$ }
   {\For{ all direct child nodes $k=1,2,...,K_i$ of $Q$ in parallel}
       {
        $(\bigtriangleup \balpha_{[Q;k]}, \bigtriangleup\bw_{k})$ $\leftarrow$ GDCA-Tree($\balpha_{[Q;k]}^{(t-1)}, \bw^{(t-1)}$) \par
        $\balpha_{[Q;k]}^{(t)}$ $\leftarrow$ $\balpha_{[Q;k]}^{(t-1)} + \beta_k\bigtriangleup\balpha_{[Q;k]}$ \par
       }%
       $\bw^{(t)}$ $\leftarrow$ $\bw^{(t-1)} +  \sum_{k=1}^{K_i} \beta_k \bigtriangleup \bw_k$ \par
} %
  \KwOut{$\bigtriangleup\balpha_{Q} \triangleq \balpha_{{Q}}^{(T_i)}-\balpha_{{Q}}^{(0)}$, and $\bigtriangleup \bw_{Q} \triangleq \bA_{Q} \bigtriangleup \balpha_{Q}$  }
}%
\end{algorithm2e}

\begin{procedure}[t]
\caption{P(). Generalized Distributed Dual Coordinate Ascent  (GDCA-Tree) for a leaf node $Q$ on the layer-$p$}
  \label{alg:leaf}
  \SetAlgoLined
{\small
   \KwIn{ $T_p \geq 1$, $\balpha_{Q} \in \R^{|Q|}$, and $\bw \in \R^{d}$ consistent with other coordinate blocks of $\balpha$ s.t. $\bw=\bA\balpha$ }
   \textbf{Data:} ${\{(\bx_i, y_i)\}}_{i \in Q}$  \par
   \textbf{Initialization:} $\bigtriangleup \balpha_{Q} \leftarrow 0 \in \R^{|Q|}$, and $\bw^{(0)} \leftarrow \bw$ \par
   \For { $h=1$  \KwTo $T_p$ }
   {    choose $i \in Q$ uniformly at random \par
        find $\bigtriangleup \alpha$ maximizing $-\frac{\lambda m}{2} ||\bw^{(h-1)} + \frac{1}{\lambda m} \bigtriangleup\alpha \bx_i||^2 - \ell_i^{*}(-(\alpha_i^{(h-1)} + \bigtriangleup\alpha))$ \par
        $\alpha_{i}^{(h)}$ $\leftarrow$ $\alpha_{i}^{(h-1)} + \bigtriangleup \alpha$ \par
        $ (\bigtriangleup \balpha_{Q})_i$ $\leftarrow$ $(\bigtriangleup\balpha_{Q})_i + \bigtriangleup\alpha$ \par
	$\bw^{(h)}$ $\leftarrow$ $\bw^{(h-1)} + \frac{1}{\lambda m} \bigtriangleup\alpha \bx_i$ \par
   }
      \KwOut{ $\bigtriangleup\balpha_{Q}$ and $\bigtriangleup\bw_{Q} \triangleq \bA_{Q} \bigtriangleup\balpha_{Q}$ }
}%
\end{procedure}
From Proposition \ref{prop:LocalSDCA} and Assumption \ref{asp:Local_Improve}, over randomly chosen data points in the LocalSDCA at a leaf node, we can have guaranteed improvement at its parent's node in terms of the expectation of local sub-optimality gap. 

For Algorithm \ref{alg:general_weighted}, we have the following convergence theorem at a general tree node $Q$:
\begin{theorem}
\label{thm:gen_GDCA}
For a tree node $Q$ on the $i$-th layer, $i=0, ..., p-1$, having $K_i$ direct child nodes satisfying Assumption \ref{asp:Local_Improve}, with parameters $\Theta^{1}_{i+1}$, $\Theta_{i+1}^2$, ..., and $\Theta^{K_i}_{i+1}$. Assume that loss functions $\ell_i(\cdot)$'s are $\nicefrac{1}{\gamma}$-smooth. Then for any input $\bw$ to Algorithm \ref{alg:general_weighted} with $T_i$ iterations, the following geometric convergence rate holds for $Q$:
\par\noindent\small
\begin{align}
\label{ieq:convergence_GDCA}
    & \E [D(\balpha_Q^{*}, \balpha_{\overline{Q}}) - D(\balpha_Q^{(T_i)}, \balpha_{\overline{Q}}) ] \\
    & \hspace{-0.3em} \leq \underbrace{ \bigg( \max_{k=1,...,K_i}\hspace{-0.3em} \big( \hspace{-0.1em} 1 \hspace{-0.2em} - \hspace{-0.2em} (1 \hspace{-0.2em} -\hspace{-0.2em} \Theta^k_{i+1}) \beta_{k} \big) \frac{\lambda m \gamma}{ \rho_i \hspace{-0.2em} + \hspace{-0.2em} \lambda m \gamma}\bigg)^{T_i}}_{\text{Convergence bound}} \nonumber \\
    &\quad\quad\quad \times \big(D(\bm{\alpha}_Q^{*}, \bm{\alpha}_{\overline{Q}}) \hspace{-0.2em} -\hspace{-0.2em} D(\balpha_Q^{(0)} \hspace{-0.3em} ,\balpha_{\overline{Q}}) \big),\nonumber 
\end{align}
\normalsize
where $\rho_i$ is any real number larger than $\rho_{min}$ defined as 
\par\noindent\small
\begin{align*}
    \rho_{min} \hspace{-0.2em}\triangleq \hspace{-0.2em}\underset{\balpha_Q \in \mathbb{R}^{|Q|}}{\text{maximize}}\;\lambda^2 m^2 \frac{\sum_{k=1}^{K_i} \hspace{-0.2em}\|\bA_{[Q;k]}\balpha_{[Q;k]} \|_2^2\hspace{-0.2em} - \hspace{-0.2em} \|\bA_{Q}\balpha_{Q} \|_2^2  }{ \|\balpha_{Q} \|_2^2} \hspace{-0.2em}\geq\hspace{-0.2em} 0.
\end{align*}
\normalsize
\end{theorem}
\begin{proof}
The dual objective value for a tree node $Q$ that has $K$ direct child nodes is bounded by the following equation: \\[-15pt] 
\par\noindent\small
\begin{align}
D\big(\balpha^{(t+1)}_{[1:K]}, \balpha_{\overline{Q}} \big ) 
& = D\big(\balpha^{(t)}_{[1:K]} + \sum_{k=1}^K \beta_k \bigtriangleup \balpha_{<[k]>}, \balpha_{\overline{Q}} \big)  \nonumber\\
& \geq \sum_{k=1}^K \beta_k  D\big(\balpha^{(t)}_{[1:K]} + \bigtriangleup \balpha_{<[k]>}, \balpha_{\overline{Q}}\big),\label{ieq:jensen} \\[-20pt] \nonumber
\end{align}
\normalsize
where $\balpha_{<[k]>}$ is $\balpha_{[k]}$ with zero-padding to increase the dimension, $Q = [1:K] = \cup_{k=1}^K [k]$ is the index set corresponding to the node $Q$, and $\balpha_{\overline{Q}}$ is the un-updated coordinate. The inequality in \eqref{ieq:jensen} is obtained by using Jensen's inequality. Then, we have\\[-15pt] 
\par\noindent\small
\begin{align*}
    & \E \bigg[ D\big(\balpha^{(t+1)}_{[1:K]}, \balpha_{\overline{Q}}\big) - D\big(\balpha^{(t)}_{[1:K]},\balpha_{\overline{Q}}\big) \bigg]  \\
    & \geq \E \bigg[ \sum_{k=1}^{K} \beta_k \big( D\big(\balpha^{(t)}_{[1:K]} + \bigtriangleup \balpha_{<[k]>},\balpha_{\overline{Q}}\big) - D\big(\balpha^{(t)}_{[1:K]},\balpha_{\overline{Q}}\big)   \big) \bigg]\\
    & = \E \bigg[ \sum_{k=1}^{K} \beta_k \big( \epsilon_{Q,k}(\balpha^{(t)}_{[1:K]},\balpha_{\overline{Q}}) - \epsilon_{Q,k}(\balpha^{(t)}_{[1:K]} + \bigtriangleup \balpha_{<[k]>},\balpha_{\overline{Q}}) \big) \bigg]\\
    & = \sum_{k=1}^{K} \beta_k \big( \E [ \epsilon_{Q,k}(\balpha^{(t)}_{[1:K]},\balpha_{\overline{Q}})] - \E[ \epsilon_{Q,k}(\balpha^{(t)}_{[1:K]} + \bigtriangleup \balpha_{<[k]>},\balpha_{\overline{Q}}) ] \big)\\
    & \geq \sum_{k=1}^{K} \beta_k (1-\Theta^k) \big( \epsilon_{Q,k}(\balpha^{(t)}_{[1:K]},\balpha_{\overline{Q}}) \big)\\
    & \geq \bigg( \min_{k=1,2,...,K} \underbrace{ (1-\Theta^k)\beta_k }_{\substack{\text{imbalanced data factor}\\{\text{and its compensation parameter}}}} \bigg) \underbrace{\sum_{k=1}^{K} \big( \epsilon_{Q,k}(\balpha^{(t)}_{[1:K]},\balpha_{\overline{Q}}) \big) }_{(A)}
\end{align*}
\normalsize
The lower-bound of $(A)$ can be obtained in Appendix A of \cite{cho2021distributed}, which leads to \eqref{ieq:convergence_GDCA}. Due to space limitation, we omit the remaining proof.
\end{proof}

%

This theorem indicates that for a tree node $Q$, if its child nodes satisfy the local sub-optimality gap with local improvement, then, at the node $Q$, GDCA-Tree can have the convergence rate shown in \eqref{ieq:convergence_GDCA}. From Proposition \ref{prop:LocalSDCA}, we know that Assumption \ref{asp:Local_Improve} holds at leaf nodes using LocalSDCA. Then, for the convergence of GDCA-Tree in a general tree network, the left hand side of \eqref{ieq:convergence_GDCA} can be considered as the local sub-optimality gap that the tree node $Q$ can achieve from $(\balpha^{(T_i)}_Q, \balpha_{\overline{Q}})$, and the convergence bound in \eqref{ieq:convergence_GDCA} can be thought of as local improvement $\Theta$ from the parent node of $Q$. In this way, the convergence of GDCA-Tree in a whole tree network can be understood in a recursive manner.

\subsection{Impact of imbalanced data on the convergence speed}
\vspace{-0.5em}
From Proposition \ref{prop:LocalSDCA}, among all leaf nodes, with the same number of local iterations, $T_p$, in LocalSDCA, the leaf node having the largest number of data points will have the largest $\Theta_p$ (i.e., close to 1) due to $\nicefrac{1}{m_B}$ term in \eqref{eq:Theta}, and become a bottleneck in the convergence speed. Thus, $\nicefrac{1}{m_B}$ in \eqref{eq:Theta} (or broadly $(1-\Theta^k)$) can be thought of as imbalanced data factor in a leaf node, while $\beta_k$ can be considered as a compensation parameter for the imbalanced data factor at the $k$-th leaf node by providing more weights on a bottleneck node. Note that in the case of balanced data, $\beta_k = \nicefrac{1}{K}$, $k=1,...,K$.

Furthermore, the parameter $\rho_{min}$ is a value indicating the similarity of global parameter $\bw$'s among the $K$ child nodes. If we have a smaller $\rho$ value, i.e., smaller similarity between the global parameters $\bw$ in child nodes, then we will have a larger convergence bound. In other words, as the global parameters among child nodes become similar, the convergence speed will decrease.

\vspace{-0.5em}
\subsection{Determining the compensation parameter $\beta_k$ for imbalanced data}\label{subsec:beta}
\vspace{-0.5em}
In order to determine the weight parameter $\beta_k$ that compensates imbalanced data effect, we consider the case where each node on a network has different numbers of data points. Suppose we have a general node $Q$ which has $K$ direct child nodes. The $k$-th child node has a partial dataset, where the set of indices of data points is denoted by $[Q;k]$. Due to the scenario that we consider here, the cardinality of the set $[Q;k]$, i.e., the number of data points in the $k$-th direct child node of $Q$, is different from each other. By considering the imbalanced numbers of data scenario, we propose the following weight calculation for $\beta_k$ in the weighted sum updating scheme \eqref{eq:updating}:\\[-25pt]
\begin{align}\label{eq:betak_scenario1}
&\beta_k = \nicefrac{|[Q;k]|}{|Q|},\\[-20pt] \nonumber
\end{align}
where the cardinality of the set of indices of data points in the $k$-th direct child node of $Q$ is represented by $|[Q;k]|$. The weighted sum updating scheme is used to put more weight on local updating parameters obtained from processing more data.  The intuition behind this is that if one local worker (e.g., $W_A$) has most data and other workers (e.g., $W_B$) have only a few, then, the solution obtained from $W_A$ is prone to be closer to the global solution $\bw^{\star}$. Thus, the global parameter $\bw^{(t)}$ is updated based on the portion of data in each local worker or sub-central node. The computation of $\beta_k$ can be done as a preprocessing step and is negligible when compared to the coordinate ascent operation.

\vspace{-0.5em}
\subsection{Number of iterations considering imbalanced data}
\label{subsec:Delayed_DDCA}
\vspace{-0.5em}
The convergence bound in \eqref{ieq:convergence_GDCA} is a function of the number of iterations as well as the imbalanced data factor. Therefore, the question raised is that is there any relationship between the number of local iterations and the imbalanced data? Let us answer this question here. For clarity and simplicity, we consider a cluster shown in Fig. \ref{fig:generalDistSystem}, where the sub-central node is denoted by $Q$. For the optimal number of local iterations to have minimum execution time in convergence, as in \cite{cho2021distributed}, we consider the following optimization problem:
\par\noindent\small
\begin{align}\label{eq:optimal_T}
\underset{T_p \geq 0}{\text{minimize}}\;\bigg( \hspace{-0.1em} 1 \hspace{-0.2em} - \hspace{-0.2em} (1 \hspace{-0.2em} -\hspace{-0.2em} \Theta^k_{p}) \beta_{k} \frac{\lambda m \gamma}{ \rho_i \hspace{-0.2em} + \hspace{-0.2em} \lambda m \gamma}\bigg)^{T_{p-1}},
\end{align}
\normalsize
where $T_p$ and $T_{p-1}$ are the numbers of local iterations in a leaf node and its parent node respectively. $\Theta_p$ is introduced in \eqref{eq:Theta}. Since the total time, $t_{total}$, for distributed ML process in a cluster can be expressed as $t_{total} = (t_{lp} T_p + t_{delay} + t_{cp}) \cdot T_{p-1}$, where $t_{lp}$ is local processing time per one iteration in a leaf node, $t_{delay}$ is time for communication delay in round trip between a leaf node and its parent node, i.e., sub-central node in Fig. \ref{fig:generalDistSystem}, and $t_{cp}$ is time for accumulation at a sub-central node. Then, by plugging \eqref{eq:Theta} into \eqref{eq:optimal_T}, denoting $\frac{\lambda m \gamma}{ 1+ \lambda m \gamma}$ as $c_1$, and $\frac{\lambda m \gamma}{\rho_i + \lambda m \gamma}$ as $c_2$, we have
\par\noindent\small
\begin{align*}
\underset{T_p \geq 0}{\text{minimize}}\;\bigg( \hspace{-0.1em} 1 \hspace{-0.2em} - \hspace{-0.2em} (1 \hspace{-0.2em} - (1- \frac{c_1}{|[Q;k]|})^{T_p}) c_2\cdot \frac{|[Q;k]|}{|Q|} \bigg)^{\frac{t_{total}}{t_{lp} T_p + t_{delay} + t_{cp}}}.
\end{align*}
\normalsize

From \cite{cho2021distributed}, the optimal number of local iterations $T_p$ at the leaf node can be obtained as follows:
\par\noindent\small
\begin{align*}
T_p = \frac{1}{\ln (1 - \frac{c_1}{|[Q;k]|})} W \bigg(  (1- \frac{c_1}{|[Q;k]|} )^r \ln ( 1- c_2\cdot \frac{|[Q;k]|}{|Q|})\bigg) - r,
\end{align*}
\normalsize
where $c_1, c_2\in [0,1)$, communication delay severity level between the local processing time and the communcation delay is denoted by $r$, i.e., $r=\nicefrac{(t_{delay}+t_{cp})}{t_{lp}}$, and $W(\cdot)$ is the Lambert W-function \cite{corless1996lambertw}. If the delay severity $r$ is zero or small enough to ignore, then, the Lambert W-function term can be expressed as a small negative constant, and $\ln(1-x) \approx -x$ for small $x$, we have 
\par\noindent\small
\begin{align} \label{eq:opt_Tp}
T_p  \propto |[Q;k]|.
\end{align}
\normalsize
Namely, the number of local iterations need to be increased proportionally to the number of data points in a node in order to minimize the execution time during convergence. Therefore, by taking into account the size of the data in the bottleneck node, we propose the delayed GDCA-Tree, where the number of local iterations is determined by the size of the local dataset in a bottleneck node. In the delayed GDCA-Tree method, the sharing of information between local workers and the sub-central node is delayed, but with a larger number of local iterations. This results in faster convergence speed.


\vspace{-0.5em}
\section{Numerical Experiments}
\label{sec:simulation}
\vspace{-0.5em}
To validate the performance of the delayed GDCA-Tree, we conduct simulations and compare it to the standard DDCA-Tree when dealing with imbalanced data on a tree network. We test the method on various machine learning tasks including regression and classification using the following datasets: wine quality dataset\footnote{\url{https://archive.ics.uci.edu/ml/datasets/wine+quality}} and Covtype dataset\footnote{\url{https://www.csie.ntu.edu.tw/~cjlin/libsvmtools/datasets/binary.html\#covtype.binary}}.

In the wine quality dataset, there are 6493 data points with 12 attributes. The twelfth attribute, quality, is used as the measurement $\by$. Each instance is normalized with $\ell_2$ norm to ensure $\| \bx_i \|_2 \leq 1$, $i=1,...,m$. To distribute the dataset over a tree network, we organize a two-layered tree network with one central node, two sub-central nodes (denoted by $S_1$ and $S_2$), and four local workers (denoted by $W_i$, $i=1,..., 4$). Each sub-central node has two local workers. The tuning parameter $\lambda$ in \eqref{prob:dual} is set to 1. We consider the case where each local worker has a different number of data points. To simulate this scenario, we unevenly distribute the wine quality data into four local workers. 30\% of the total dataset are allocated to three local workers, $W_1$, $W_2$, and $W_3$ (10\% each, which is 649 data points). The remaining 70\% of the total dataset (4546 data points) is allocated to one local worker, i.e., $W_4$, without overlapping the data among local workers.

In the delayed GDCA-Tree, we use \eqref{eq:betak_scenario1} to calculate $\beta_k$ for each local worker. Therefore, for each updating global parameter $\bigtriangleup \bw_k$ that each local worker has from its local dataset, $\beta_1 = \beta_2 = 0.5 (=649/(649+649))$, $\beta_3 = 0.1249 (=649/(649+4546))$, and $\beta_4 = 0.8752 (=4546/(649+4546))$. At the sub-central nodes $S_1$ and $S_2$, $0.2 (= (649+649)/6493)$ and $0.8 (=(649+4546)/6493)$ are used for $\beta_k$ respectively.

In order to obtain statistical results, we run 100 random trials. From the 100 trials, we calculate the average execution time at each outer iteration and compute the average duality gap $P(\bw(\balpha^{(t)})) - D(\balpha^{(t)})$ at the central station. For the number of local iterations $T_p$ in Procedure \ref{alg:leaf},  we use $T_p = 100$. And for delayed GDCA-Tree, we use 300 local iterations for $T_p$ by considering the increased number of data points in $W_4$. As shown in Fig. \ref{fig:sim_scenario1}(a), the delayed GDCA-Tree (red solid line) can improve the convergence speed, compared to the standard DDCA-Tree \cite{cho2021distributed} with imbalanced data (blue dotted line). The figure illustrates that by using the delayed GDCA-Tree, we can mitigate the effect of imbalanced data on distributed dual coordinate ascent in a tree network.

\begin{figure}[t]
    \centering
   \subfloat[Wine quality]{ \includegraphics[scale=0.45]{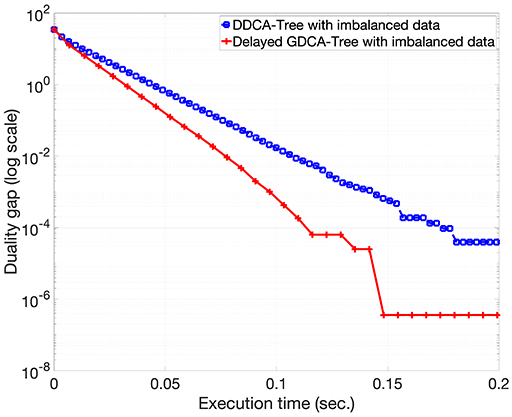} }\;
   \subfloat[Covtype]{\includegraphics[scale=0.45]{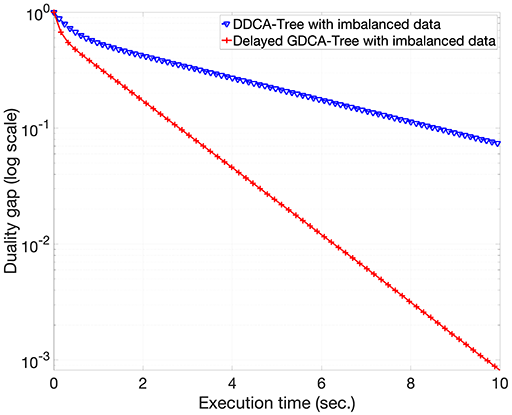}}
   \vspace{-0.5em}
    \caption{\small{Comparison between standard DDCA-Tree (blue)  and delayed GDCA-Tree (red) with imbalanced data.}}
    \label{fig:sim_scenario1}
\end{figure}

Additionally, we run binary classification with Covtype dataset having a total of 581012 instances and 12 attributes. The dataset is normalized and the labels are set to be in the set \{-1,1\}. A standard hinge loss and $\ell_2$ regularization are used for a linear SVM. The tree network used has one central node, two sub-central nodes, and eight local workers, where each sub-central node has four local workers. The data is distributed in an imbalanced way, with 5\% of the total dataset going to seven local workers, and 65\% going to one local worker with no overlap. The number of communications between the local workers and the sub-central node is set to 10. The numbers of local iterations in local workers in the standard DDCA-Tree and the delayed GDCA-Tree are set to 1000 and 4000 respectively. Fig. \ref{fig:sim_scenario1}(b) demonstrates that under the imbalanced data scenario, the delayed GDCA-Tree can improve the convergence speed of the DDCA-Tree by considering the information of imbalanced data.

\vspace{-0.5em}
\small
\bibliographystyle{IEEEbib}
\bibliography{refs_distributedML,refs_GSDCA}

\end{document}